%% file: main.tex
\documentclass[twocolumn]{scholar7}

\usepackage{microtype}
\usepackage{graphicx}
\usepackage{subcaption}
\usepackage{booktabs} 
\usepackage{tcolorbox}
\tcbuselibrary{skins}

\usepackage{tikz}
\usepackage{pgfplots}
\usepgfplotslibrary{fillbetween}
\pgfplotsset{compat=1.18}
\usetikzlibrary{arrows.meta, positioning, calc, shapes, decorations.pathreplacing}

\usepackage{hyperref}

\usepackage{titletoc}
\usepackage[subfigure]{tocloft}
\usepackage{amsmath}
\usepackage{amssymb}
\usepackage{mathtools}
\usepackage{amsthm}

\theoremstyle{plain}
\newtheorem{theorem}{Theorem}[section]
\newtheorem{proposition}[theorem]{Proposition}

\newtheorem{corollary}[theorem]{Corollary}
\theoremstyle{definition}
\newtheorem{definition}[theorem]{Definition}
\newtheorem{assumption}[theorem]{Assumption}
\theoremstyle{remark}

\newcommand{\Corr}{\mathrm{Corr}}

\newcommand{\Truth}{T}
\newcommand{\Proxy}{P}
\newcommand{\Verify}{V}
\newcommand{\Score}{S}

\newcommand{\tc}{\rho}

\newcommand{\vpressure}{\Lambda}
\newcommand{\snr}{r}

\newcommand{\Ce}{C_{\mathrm{eff}}}
\newcommand{\Be}{\mathcal{B}_{\mathrm{eff}}}

\title{Preventing the Collapse of Peer Review Requires Verification-First AI}

\author[1,2]{Lei You}
\author[2]{Lele Cao}
\author[3,4]{Iryna Gurevych}

\affiliation[1]{Technical University of Denmark} 
\affiliation[2]{CSPaper @ Scholar7}
\affiliation[3]{Technical University Darmstadt}
\affiliation[4]{MBZUAI}

\correspondence{\email{leiyo@dtu.dk}, \email{lei@scholar7.com}}
\date{January 2026}

\abstract{%
This paper argues that AI-assisted peer review should be verification-first rather than review-mimicking.
We propose \emph{truth-coupling}, i.e. how tightly venue scores track latent scientific truth, as the right objective for review tools.
We formalize two forces that drive a phase transition toward proxy-sovereign evaluation: \emph{verification pressure}, when claims outpace verification capacity, and \emph{signal shrinkage}, when real improvements become hard to separate from noise.
In a minimal model that mixes occasional high-fidelity checks with frequent proxy judgment, we derive an explicit coupling law and an incentive-collapse condition under which rational effort shifts from truth-seeking to proxy optimization, even when current decisions still appear reliable.
These results motivate actions for tool builders and program chairs: deploy AI as an adversarial auditor that generates auditable verification artifacts and expands effective verification bandwidth, rather than as a score predictor that amplifies claim inflation.
}

\begin{document}

\maketitle

\input{sections/01_introduction.tex}

\input{sections/02_universes.tex}

\input{sections/04_model.tex}
\input{sections/05_theory.tex}
\input{sections/06_ai_peer_review.tex}
\input{sections/07_alternative_views.tex}
\input{sections/08_call_to_action.tex}

\input{sections/10_discussion.tex}

\begin{appendix}
\onecolumn

\par\vspace{0.2cm}
\begin{center}
    {\Large \bf Appendix: Preventing the Collapse of Peer Review Requires Verification-First AI}
\end{center}
\par\vspace{0.2cm}

\par\vspace{1em}
{\small
\setcounter{tocdepth}{1}
\startcontents[sections]
\printcontents[sections]{l}{1}{}
}

\input{appendices/appendix_overview.tex}
\input{appendices/03_related_work}
\input{appendices/appendix_proofs.tex}
\input{appendices/appendix_goodhart_amplification.tex}
\input{appendices/appendix_measuring_saturation.tex}

\input{appendices/appendix_H_citations_proxy}
\input{appendices/09_empirical_agenda.tex}
\input{appendices/appendix_case_studies.tex}
\input{appendices/appendix_limitations}

\end{appendix}

\bibliographystyle{plainnat}
\bibliography{references}

\end{document}

%% file: sections/01_introduction.tex
\section{Introduction}

Peer review and venue selection act as the main feedback channel that couples scientific work to recognition and career outcomes.
This channel now operates under stress.
Submission volumes keep rising, reviewer time does not scale proportionally, and large language models (LLMs) reduce the marginal cost of producing paper-like text.
Recent work has explored LLMs as review generators, score predictors, or meta-review assistants \citep{zhou-etal-2024-llm,thelwall2025predictive,shin2025automatic,bao2021acceptance}.
A natural ``default goal'' in this line of work is imitation: we train or prompt a model to produce review text that reads like a human review, or to predict a human score.

Our analysis suggests that we should use AI-assisted peer review to increase truth-coupling by reducing \emph{verification friction} and expanding effective \emph{verification bandwidth}. We should \textbf{not} treat agreement with human review text or scores as the main objective, \textbf{even if} they come from a venue with historically low noise and high trustworthiness. A feasibility demonstrated in a recent work on automating paper--code consistency checks \cite{scicoqa2026}, complemented by efforts to use LLMs for review assistance and systematic error identification in papers \cite{liu2023reviewergpt,xi2025flaws}.


We argue that imitation targets the wrong object: for a fixed paper, a human review score is not a gold label for truth but an outcome of a bandwidth-constrained social process that relies on heuristics and proxy signals \citep{shah2024survey}.
A concrete warning sign is a NeurIPS 2024 experiment where an LLM-based checklist assistant improved compliance yet could be gamed by fabricated justifications \citep{goldberg2024checklist}, illustrating a broader pattern that once proxy signals become objectives, optimization shifts to proxies rather than what they were meant to measure \citep{goodhart1975problems,manheim2019categorizing,campbell1979assessing}.

\subsection{A Phase Transition View of Scientific Production}

We formalize a split that many researchers informally recognize.
Some subfields feel like \emph{truth-sovereign} regimes.
In these regimes, a small number of decisive checks, replications, or proofs can settle disputes, and recognition tends to track what is correct over time.
Other subfields feel like \emph{proxy-sovereign} regimes.
In these regimes, claims arrive faster than the community can verify them, so the evaluation process leans on easier-to-check signals, namely benchmark scores, formal compliance, fashionable phrasing, or dense-looking technicality.

We model this split as a phase transition driven by two forces.
First, \emph{verification friction} captures how costly it is for a community to produce decisive evidence per claim.
Second, \emph{signal shrinkage} captures how small true improvements become as a domain approaches an S-curve \citep{bloom2020ideas}.
When the true signal shrinks, even honest progress becomes a low signal-to-noise problem. Proxy signals turn more attractive as they are easier to produce and evaluate.

\subsection{Three Implications That Are Easy to Miss}
\label{sec:three-implications}

Our model yields three implications that are easy to miss. 

\noindent\textbf{Implication 1 (High prestige is a deceptive lagging indicator).}
A scientific venue can appear statistically healthy while being structurally doomed.
Our analysis shows that "high prestige" is merely a historical record of low verification pressure.
Therefore, fitting AI agents to mimic top-tier venues merely replicates past signals without inheriting the verification capacity required to withstand the coming surge in AI-generated claims.

\noindent\textbf{Implication 2 (Review-mimicking automates the collapse).} Training prompting AI to predict human scores scales the \emph{output} of evaluation (scores) without scaling the \emph{capacity} for scrutiny (verification). This does not solve the review crisis; it hides it. By decoupling ``getting a high score'' from ``passing a hard check,'' mimicking tools effectively remove the only friction that stopped authors from gaming the system, accelerating the drift into ``Zombie Science''.

\noindent\textbf{Implication 3 (The goal is pressure control, not score optimization).}
To survive the AI era, peer review must shift its objective from ``predicting quality'' to ``controlling verification pressure''.
The only legitimate role for AI is to act as an adversarial auditor that physically expands effective bandwidth by executing evidence. Any tool that merely generates text or scores without checking the underlying reality is simply inflating the bubble.

\subsection{What We Contribute}

We make four contributions.
First, we define verification pressure and signal shrinkage as the two main drivers of a phase transition from truth sovereignty to proxy sovereignty.
Second, we derive a closed-form expression for truth-coupling $\tc=\Corr(\Score,\Truth)$ under a simple but interpretable mixture model of verification and proxy evaluation.
Third, we analyze researcher incentives and show a collapse condition under which truth effort becomes privately irrational.
Fourth, we translate these results into design principles for AI-assisted peer review, and we outline an empirical agenda for measuring the key quantities in real venues.

\subsection{Roadmap}

Section~\ref{sec:universes} grounds the model in a sociological description of two regimes.
Sections~\ref{sec:model} and \ref{sec:theory} introduce the formal model and the main results.
Section~\ref{sec:ai} translates the model into verification-first design principles for AI-assisted review.
Section~\ref{sec:altviews} describes credible alternative views and explains where we agree and disagree.
Section~\ref{sec:call} synthesizes these findings into concrete policy recommendations.
Section~\ref{sec:conclusions} concludes the paper.
Appendix~\ref{sec:related} reviews relevant theory on metrics, incentives, and peer review automation.

%% file: sections/02_universes.tex
\section{Universe A and Universe B}
\label{sec:universes}

We use the terms \emph{Universe A} and \emph{Universe B} as shorthand for two regimes that can coexist inside the same conference or department of a university.
The labels are not moral categories.
They describe two different ways a field closes its feedback loop between work and recognition.

\subsection{Two Kinds of Sovereignty}

In Universe~A, \emph{truth is sovereign}.
A claim eventually lives or dies based on a relatively decisive interaction with the target of study, such as a proof, a falsifiable experiment, or a replication with high fidelity.
Human judgment still matters, but it matters mostly for allocating attention, not for deciding what is true.

In Universe~B, \emph{proxy signals are sovereign}.
Claims arrive faster than the community can verify them, so the system relies on low-cost signals that correlate with truth only weakly.
Examples include small benchmark gains, compliance checklists, citation cascades, or the appearance of technical depth.
When such signals decide careers, researchers rationally learn to optimize them. We treat citation cascades as one concrete example of a low-cost proxy signal, and we analyze this case in Appendices~\ref{app:citations-proxy}.

\subsection{Career Dynamics as An Observable Signature}

The two regimes also differ in how they shape careers.
In Universe~A, the main scarce resource is insight.
Insight is hard to outsource and hard to parallelize.
As a result, even senior researchers often stay close to the technical frontier.
A principal investigator (PI) may supervise students, but the PI usually keeps strong ownership of the key ideas and can personally verify core arguments.

In Universe~B, the main scarce resource is attention and throughput.
Outputs aggregate linearly across people, and the evaluation system rewards volume.
As a result, career success favors management of a large production pipeline.
Many successful PIs in Universe~B shift away from hands-on verification and toward fundraising, staffing, and strategic positioning.
This behavior is not a personal failure.
It is a predictable response to incentives in a proxy-sovereign regime \citep{holmstrom1991multitask,power1997audit}.

\subsection{Verification Friction and Evidence Compressibility}

A common confusion is to equate ``easy to verify'' with ``cheap to verify.''
Some of the most truth-sovereign disciplines have high absolute verification cost.
A major mathematical proof may take months of expert effort to check.
What matters is not the raw cost $c$ but whether the evidence is \emph{decisive and compressible}.

We describe this using an \emph{effective verification cost}
\begin{equation}
\label{eq:Ceff}
\Ce = \frac{c}{\kappa},
\end{equation}
where $\kappa \in (0,1]$ denotes the fidelity or decisiveness of the evidence.
A careful proof has large $c$ but often has $\kappa$ close to one.
A noisy empirical study can have large $c$ and small $\kappa$, which produces a large $\Ce$.
A simulation study can have small $c$ but also small $\kappa$ when the simulator fails to capture the real-world mechanism of interest.
In that case, $\Ce$ can still be large.
This is why proxy sovereignty often emerges in simulator-grounded mature engineering domains.

\subsection{Signal Shrinkage and Saturation}

A second driver of proxy sovereignty is the shrinkage of true effect sizes.
As domains mature, improvements become smaller and harder to measure \citep{bloom2020ideas}.
When the between-paper variance in true value shrinks, reviewers face a low signal-to-noise problem even if they act in good faith.
Proxy signals then become tempting because they are easier to compare at scale.

Appendix~\ref{app:cases} provides additional case studies 
and maps them to distinct regions of Universe A and Universe B.


%% file: sections/04_model.tex
\section{Model}
\label{sec:model}

We propose a stylized model that keeps two features explicit.
First, verification is scarce.
Second, as fields saturate, true improvements become harder to distinguish from noise.

\subsection{Latent Truth, Proxy Judgment, and Score}

Each submission has an underlying (possibly multi-dimensional) scientific value; we model the venue-relevant scalar projection as $T$.
In mathematics it corresponds to correctness and conceptual novelty.
In empirical work it corresponds to validity, robustness, and generality.
We treat $\Truth$ as a random variable with $\Var(\Truth)>0$ across the population of submissions to a venue in a fixed period.

A venue assigns each submission a scalar decision signal $S$  that influences acceptance and future allocation of attention.

\begin{assumption}[Mixture evaluation model]
\label{assump:mixture}
With probability $q \in [0,1]$, reviewers obtain verification evidence and assign a score based on $\Truth$.
With probability $1-q$, reviewers rely on a proxy judgment that equals $\Truth+\Delta$, where $\Delta$ is a zero-mean proxy error.\footnote{The zero-mean assumption is across submissions: decisions are made once per paper, so proxy noise does not average out.} 
We assume $\Delta$ is independent of $\Truth$.
\end{assumption}

Under Assumption~\ref{assump:mixture}, we can write
\begin{align}
\label{eq:score}
\Score & = q \Truth + (1-q) (\Truth + \Delta) \nonumber \\
       & = \Truth + (1-q)\Delta.
\end{align}
This expression absorbs many details into $\Delta$.
It captures random proxy noise, model instability, benchmark overfitting, and strategic ``gaming'' of evaluation criteria.
We treat $q$ as an \emph{endogenous} quantity determined by verification capacity, which we formalize next.

\begin{definition}[Truth-coupling]
\label{def:tc}
We define truth-coupling as $\tc \coloneqq \Corr(\Score,\Truth)$.
\end{definition}

Truth-coupling is high when a venue's selection process preserves a strong relationship between recognition and latent value.
Truth-coupling is low when proxy errors dominate.

\subsection{Verification Pressure}

We describe the verification bottleneck using three quantities.
The claim rate $R$ denotes the number of claims that a venue processes per unit time.
This quantity grows with submission volume and with the ease of producing plausible claims.
The effective verification cost $\Ce$ denotes the cost of producing decisive evidence per claim, as in \eqref{eq:Ceff}.
The effective verification bandwidth $\Be$ denotes how much decisive evidence the community can process per unit time.
This term bundles reviewer time, reviewer expertise, and the availability of external checks such as replications, proof checking, or artifact evaluation.

We define verification pressure as
\begin{equation}
\label{eq:Lambda}
\vpressure \coloneqq \frac{R\,\Ce}{\Be}.
\end{equation}
When $\vpressure \le 1$, the community can in principle verify a large fraction of claims.
When $\vpressure$ grows above one, verification becomes selective, and $q$ decreases.

\begin{assumption}[Verification frequency]
\label{assump:q}
The probability of verification is
\begin{equation}
\label{eq:q}
q = \min\!\left\{1,\,\frac{1}{\vpressure}\right\}
= \min\!\left\{1,\,\frac{\Be}{R\,\Ce}\right\}.
\end{equation}
\end{assumption}

This assumption treats verification as a capacity-constrained queue with random service.
It is not a detailed model of program committee behavior.
It is a transparent abstraction that makes the core dependence explicit.

\subsection{Signal Shrinkage and Noise-to-Signal Ratio}

We also need a way to represent saturation.
We use the variance of $\Truth$ across submissions as a proxy for how much true signal is available to reviewers.
In many domains, true effect sizes shrink over time \citep{bloom2020ideas}.
This shrinkage reduces $\Var(\Truth)$.\footnote{Nascent fields like XAI mirror ``pre-ImageNet'' CV: undefined $T$ and high $\Ce$ yield low $\rho$. ImageNet temporarily anchored Universe A by serving as verification infrastructure (lowering $\Ce$), until benchmark saturation and gaming ($\Var(\Delta) \uparrow$) triggered a regression to Universe B. 
}

We summarize the balance between proxy error and truth signal with a dimensionless ratio
\begin{equation}
\label{eq:snr}
\snr \coloneqq \frac{\Var(\Delta)}{\Var(\Truth)}.
\end{equation}
Large $\snr$ means proxy noise dominates.
Small $\snr$ means true value dominates.
Both verification pressure $\vpressure$ and saturation can increase $\snr$.
The former can increase $\Var(\Delta)$ through gaming and instability (Appendix~\ref{app:goodhart-amplification}), while the latter decreases $\Var(\Truth)$. Appendix~\ref{app:saturation} discusses practical proxies for signal shrinkage and sketches how we can estimate $\Var(\Truth)$ from benchmark time series and robustness reports.


%% file: sections/05_theory.tex
\section{Theory}
\label{sec:theory}

We now analyze how verification pressure $\vpressure$ and noise-to-signal ratio $\snr$ shape truth-coupling and researcher incentives.
All proofs appear in Appendix~\ref{app:proofs}.

\subsection{Truth-Coupling as a Closed Form}

Our first result links truth-coupling to $q$ and $\snr$.

\begin{theorem}[Truth-coupling under mixture evaluation]
\label{thm:tc}
Under Assumption~\ref{assump:mixture}, truth-coupling satisfies
\begin{equation}
\label{eq:tc}
\tc \;=\; \Corr(\Score,\Truth)
\;=\; \left(1 + (1-q)^2 \snr\right)^{-1/2}.
\end{equation}
\end{theorem}

Theorem~\ref{thm:tc} makes two points explicit.
First, verification scarcity hurts coupling through $(1-q)^2$.
Second, saturation hurts coupling through $\snr$.
Even if reviewers act in good faith, small true effect sizes make it hard to preserve a strong relationship between score and truth.

A direct corollary translates a coupling target into a required verification frequency.

\begin{corollary}[Coupling budget]
\label{cor:budget}
Fix a target $\tc_{\min} \in (0,1)$.
If $\snr>0$, then $\tc \ge \tc_{\min}$ holds whenever
\begin{equation}
\label{eq:qmin}
q \;\ge\; 1 - \sqrt{\frac{\tc_{\min}^{-2}-1}{\snr}}.
\end{equation}
\end{corollary}

Corollary~\ref{cor:budget} shows that coupling constraints become more demanding as $\snr$ grows.
This implies that a field saturates and true gains become smaller ($\Var(\Truth)$ shrinks), maintaining the same level of truth-coupling requires \emph{higher} verification frequency.
By Assumption~\ref{assump:q}, this translates into a bandwidth requirement $\Be \ge q_{\min} R \Ce$.

\subsection{A Phase Diagram: From Universe A to Universe B}
\label{sec:phase_diagram}

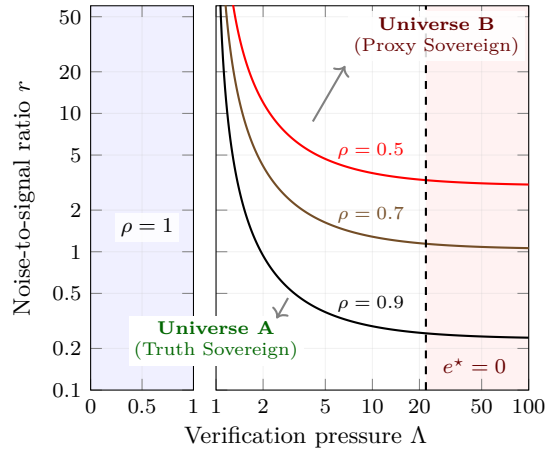
\begin{figure}[t]
\centering
\begin{tikzpicture}

\begin{axis}[
    name=axA,
    width=0.37\linewidth,
    height=0.84\linewidth,
    set layers,
    ylabel={Noise-to-signal ratio $\snr$},
    xmin=0, xmax=1,
    ymode=log,
    ymin=0.1, ymax=60,
    xtick={0,0.5,1},
    ytick={0.1,0.2,0.5,1,2,5,10,20,50},
    tick label style={font=\footnotesize},
    label style={font=\small},
    grid=both,
    minor grid style={opacity=0.2},
    major grid style={opacity=0.2},
    log ticks with fixed point,
]

\addplot[
    draw=none,
    fill=blue,
    fill opacity=0.06,
    on layer=axis background
] coordinates {(0,0.1) (1,0.1) (1,60) (0,60)} \closedcycle;

\node[
    font=\footnotesize,
    fill=white,
    fill opacity=0.85,
    text opacity=1,
    inner sep=2pt,
    align=center
] at (axis cs:0.55,1.5) {$\rho=1$};

\end{axis}

\begin{axis}[
    name=axB,
    at={(axA.east)},
    anchor=west,
    xshift=3mm,
    width=0.72\linewidth,
    height=0.84\linewidth,
    set layers,
    ylabel={},
    yticklabels=\empty,
    xmode=log,
    xmin=1, xmax=100,
    ymode=log,
    ymin=0.1, ymax=60,
    xtick={1,2,5,10,20,50,100},
    ytick={0.1,0.2,0.5,1,2,5,10,20,50},
    tick label style={font=\footnotesize},
    label style={font=\small},
    legend style={
        at={(0.02,0.98)},
        anchor=north west,
        font=\scriptsize,
        fill=white,
        fill opacity=0.85,
        draw opacity=1,
        cells={anchor=west}
    },
    grid=both,
    minor grid style={opacity=0.2},
    major grid style={opacity=0.2},
    log ticks with fixed point,
]

\def\vpc{22}
\addplot[
    draw=none,
    fill=red!5!white,
    on layer=axis background
] coordinates {(\vpc,0.1) (100,0.1) (100,60) (\vpc,60)} \closedcycle;


\addplot+[mark=none, domain=1.01:100, samples=300, thick]
    {(1/(0.5^2)-1)/((1-1/x)^2)}
    node[pos=0.81, anchor=south, yshift=1mm, font=\scriptsize, inner sep=2pt, text=.] {$\rho=0.5$};

\addplot+[mark=none, domain=1.01:100, samples=300, thick]
    {(1/(0.7^2)-1)/((1-1/x)^2)}
    node[pos=0.81, anchor=south, yshift=1mm, font=\scriptsize, inner sep=2pt, text=.] {$\rho=0.7$};

\addplot+[mark=none, domain=1.01:100, samples=300, thick]
    {(1/(0.9^2)-1)/((1-1/x)^2)}
    node[pos=0.81, anchor=south, yshift=1mm, font=\scriptsize, inner sep=2pt, text=.] {$\rho=0.9$};

\coordinate (target09) at (axis cs:3.0, 0.5);
\coordinate (target05) at (axis cs:4, 8.1);

\addplot+[mark=none, black, dashed, thick] coordinates {(\vpc,0.1) (\vpc,60)};

\node[
    font=\footnotesize,
    color=red!40!black,
    inner sep=2pt,
    align=center
] at (axis cs:45,0.15) {$e^\star=0$};

\node[
    name=labelUnivB,
    font=\scriptsize,
    color=red!40!black,
    anchor=north east,
    align=center,
    fill=white, fill opacity=0.85, text opacity=1,
    inner sep=2pt
] at (axis cs:95, 55) {\textbf{Universe~B}\\(Proxy Sovereign)};

\end{axis}

\node[
    name=labelUnivA,
    font=\scriptsize,
    color=green!40!black,
    anchor=center,          
    align=center,
    fill=white,             
    fill opacity=1,         
    inner sep=2pt,
    xshift=1.5mm,
    yshift=-1.9cm,
] at ([xshift=1.5mm, yshift=0cm]axA.east) {\textbf{Universe~A}\\(Truth Sovereign)};

\draw[->, thick, gray, shorten >=2pt, shorten <=2pt] 
    (target09) -- ([xshift=-4mm, yshift=-1mm]labelUnivA.north east);

\draw[->, thick, gray, shorten >=2pt, shorten <=2pt] 
    (target05) -- (labelUnivB.south west);

\node[font=\small, xshift=0.5cm, yshift=-0.2cm] at (current bounding box.south) {Verification pressure $\vpressure$};

\end{tikzpicture}
\caption{\textbf{Truth-coupling phase diagram.}
Contours of $\rho=\Corr(S,T)$ over verification pressure $\Lambda$ and noise-to-signal ratio $\snr$ (Theorem~\ref{thm:tc}).
For $\Lambda\le 1$ we have full verification ($q=1$) and thus $\rho=1$.
For $\Lambda>1$, truth-coupling degrades as verification becomes scarce (larger $\Lambda$) or proxy noise increases (larger $\snr$).}
\label{fig:phase}
\end{figure}

Section~\ref{sec:universes} introduced Universe~A and Universe~B as two idealized feedback regimes.
\figurename~\ref{fig:phase} makes this link concrete.
We use truth-coupling $\tc \triangleq \Corr(\Score,\Truth)$ as an order parameter.
When $\tc$ is close to $1$, evaluation is truth-sovereign and the feedback loop closes through verification.
When $\tc$ is small, evaluation is proxy-sovereign and the loop closes through cheap signals that can become targets.

The left panel shows a simple but important limit.
When $\vpressure \le 1$, verification keeps up with claims, so $q=1$.
Then the mixture model gives $\tc=1$ for any $\snr$.
In other words, if we can verify essentially everything, proxy noise does not control outcomes.

The right panel shows the regime $\vpressure>1$, where verification becomes scarce.
Here $q$ decreases as $\vpressure$ increases, so even honest reviewers must lean more on proxies.
Moving right, toward larger $\vpressure$, lowers $\tc$ because more decisions occur in the proxy mode.
Moving up, toward larger $\snr$, also lowers $\tc$ because proxy variation dominates true variation.

The dashed line is \emph{not} a boundary in $\tc$.
It marks an incentive threshold from our effort model.
To the right of this line, the researcher's optimal truth effort can hit the corner solution $e^\star=0$, even if the field currently has moderate noise.  We provide the formal derivation of this incentive collapse in Theorem~\ref{thm:effort}.


We remark that sometimes $\Var(\Delta)$ (and hence $\snr$) may itself increase with verification pressure $\vpressure$.
Appendix~\ref{app:goodhart-amplification} gives a stylized mechanism showing how selective reporting and best-of-$K$ search can make this happen.

\subsection{A Collapse Condition for Truth Effort}
\label{sec:effort-collapse}

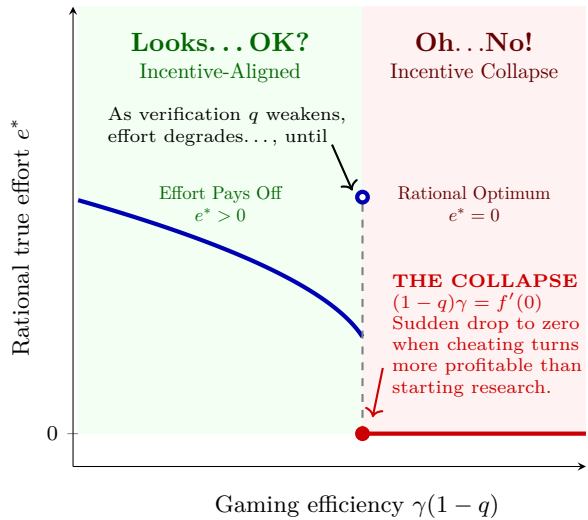
\begin{figure}[t]
    \centering
\begin{tikzpicture}
    \begin{axis}[
        axis lines = left,
        xlabel = {Gaming efficiency $\gamma(1-q)$},
        ylabel = {Rational true effort $e^*$},
        ymin = -0.1, ymax = 1.3,
        xmin = -0.05, xmax = 5,
        xtick = \empty, 
        ytick = {0},
        yticklabels = {0},
    tick label style={font=\footnotesize},
        label style={font=\small},
        xlabel style={at={(axis description cs:0.85,-0.04)},anchor=north east},
        clip=false,
        width=8.4cm,
        height=7.7cm,
        every axis plot/.append style={ultra thick},
    ]

    \def\threshold{2.8}

    \fill[green!5!white] (axis cs:0,0) rectangle (axis cs:\threshold,1.3);
    \node[green!40!black, anchor=north] at (axis cs:1.4, 1.25) {\bfseries Looks\ldots OK?};
    \node[green!40!black, scale=0.8] at (axis cs:1.4, 1.10) {Incentive-Aligned};
    \node[green!40!black, scale=0.7, align=center] at (axis cs:1.4, 0.7) {Effort Pays Off\\$e^* > 0$};

    \fill[red!5!white] (axis cs:\threshold,0) rectangle (axis cs:5,1.3);
    \node[red!40!black, anchor=north] at (axis cs:3.9, 1.25) {\bfseries Oh$\ldots$No!};
    \node[red!40!black, scale=0.8] at (axis cs:3.9, 1.10) {Incentive  Collapse};
    \node[red!40!black, scale=0.7, align=center] at (axis cs:3.9, 0.7) {Rational Optimum\\$e^* = 0$};

    
    \addplot[
        color=blue!70!black,
        domain=0:\threshold,
        samples=100
    ] {0.2 + 0.3 * sqrt(\threshold - x + 0.1)}; 

    \addplot[
        color=red!80!black,
        domain=\threshold:5,
    ] {0};

    \draw[dashed, gray, thick] (axis cs:\threshold, 0) -- (axis cs:\threshold, 0.72); 
    
    \addplot[mark=*, mark options={fill=white, draw=blue!70!black}, only marks] coordinates {(\threshold, 0.72)};
    
    \addplot[mark=*, mark options={fill=red!80!black, draw=red!80!black}, only marks] coordinates {(\threshold, 0)};

    

    \draw[->, thick, black] (axis cs:2.5, 0.88) -- (axis cs:2.7, 0.74);
    \node[anchor=south west, font=\scriptsize, align=left] at (axis cs:0.2, 0.85) 
        {As verification $q$ weakens,\\effort degrades\ldots, until};

    \draw[->, thick, red!80!black] (axis cs:3.0, 0.2) -- (axis cs:\threshold+0.1, 0.05);
    \node[anchor=west, red!80!black, font=\scriptsize, align=left] at (axis cs:3.0, 0.3) 
        {\textbf{THE COLLAPSE}\\ $(1-q)\gamma = f'(0)$ \\Sudden drop to zero\\when cheating turns\\more profitable than\\starting research.};

    \end{axis}
\end{tikzpicture}
\caption{\textbf{Incentive collapse for truth effort.}
The researcher's optimal truth effort $e^\star$ as a function of the effective proxy return $(1-q)\gamma$ (Theorem~\ref{thm:effort}).
As $(1-q)\gamma$ increases, $e^\star$ decreases and eventually hits the corner solution $e^\star=0$ when $(1-q)\gamma \ge f'(0)$.}
    \label{fig:collapse}
\end{figure}

Truth-coupling describes how well a venue's scores track latent scientific value at the community level.
We also need to understand incentives at the level of individual authors.
Even if reviewers act in good faith, authors respond to what the system rewards.

We use a minimal effort allocation model.
Each researcher chooses how to spend a unit of effort.
Truth-oriented effort supports claims that would survive careful checking.
Proxy-oriented effort improves how the work looks under cheap signals, such as marginal benchmark tuning or presentation polish.
When verification is common, proxy optimization matters less because fewer decisions rely on proxies.
When verification is scarce, proxy optimization matters more because more decisions rely on proxies.

\begin{assumption}[Effort and payoff]
\label{assump:effort}
A researcher chooses $e\in[0,1]$ to maximize
\begin{equation}
\label{eq:utility}
S(e) \;=\; f(e) \;+\; (1-q)\gamma(1-e),
\end{equation}
where $S(e)$ follows \eqref{eq:score} and $f$ captures truth-oriented production (differentiable, increasing, and concave),
and $\gamma>0$ scales the effectiveness of proxy optimization in the proxy mode.
We allow $\gamma$ to increase with tools that facilitate proxy optimization, such as automated writing.
\end{assumption}

\begin{theorem}[Truth effort decreases with verification scarcity]
\label{thm:effort}
Under Assumption~\ref{assump:effort}, any maximizer $e^\star$ satisfies the first-order condition
\begin{equation}
\label{eq:foc}
f'(e^\star) = (1-q)\gamma
\end{equation}
whenever $e^\star \in (0,1)$.
Moreover, if $(1-q)\gamma \ge f'(0)$, then an optimal solution satisfies $e^\star = 0$.
\end{theorem}

Theorem~\ref{thm:effort} delivers a brutal interpretation of the researcher's dilemma. The author faces a trade-off between the marginal return on truth, $f'(e^\star)$, and the temptation of proxy optimization, $(1-q)\gamma$. As verification bandwidth saturates ($q \to 0$) or AI tools cheapen mimicry ($\gamma \uparrow$), the incentive structure tilts aggressively against truth.

The theorem's most critical implication is the corner solution $e^\star=0$. If the return on cheap proxies exceeds the initial value of doing real science—specifically, if $(1-q)\gamma \ge f'(0)$, rational effort collapses completely, illustrated in \figurename~\ref{fig:collapse}. Crucially, this collapse condition is independent of the field's current noise level or historical reputation. Mathematically, this implies that 'top-tier' status serves merely as a historical record of low verification pressure with low noise-to-signal ratio, rather than a structural firewall against AI-generated claims.

\subsection{Theoretical Summary: The Capacity Trap}
\label{sec:theory_summary}

Our theoretical analysis leads to a sobering conclusion: under the exponential pressure of AI-driven claim inflation, \textbf{incentive collapse is not merely a risk, but the default structural destiny} of the peer review system.
We have demonstrated that truth-coupling $\rho$ acts as a deceptive lagging indicator, capable of masking deep fragility in historically trusted fields (i.e. high $\rho$ fields can also have $e^\star\approx 0$ collapse).
As long as verification bandwidth $\Be$ remains static while the volume of claims $R$ and the ease of mimicry $\gamma$ scale exponentially, the collapse condition derived in Theorem~\ref{thm:effort} becomes a mathematical inevitability.
The system is structurally condemned to drift toward Universe B.

However, the math also reveals the escape: expand the verification bandwidth ($\Be$) to reverse the collapse inequality.
This necessity, i.e. to shift from \emph{score prediction} to \emph{verification}, sets the stage for the verification-first paradigm.




%% file: sections/06_ai_peer_review.tex
\section{Implications for AI-Assisted Peer Review}
\label{sec:ai}

\begin{quote}
\itshape
``You can not stop the cold wind from blowing, but you can build a house to shelter you from it.''\\
--- Isaac Asimov, Foundation~\cite{asimov1951foundation}
\end{quote}

In our context, the ``cold wind'' is the exponential growth of AI-generated noise, and the ``house'' is the verification capacity. Just as a gas obeys the ideal gas law regardless of the motion of individual molecules, the scientific community obeys these collapse dynamics regardless of individual ethical intent. To construct this ``house'', our results in Section~\ref{sec:theory} suggest a fundamental design principle.

\medskip
\noindent\textbf{Principle (verification-first).}
\emph{We should use AI to reduce verification friction and increase truth-coupling, not to imitate human review text or to predict human scores.}
\medskip

We now justify this principle and translate it into design guidelines using the $(q,r)$-framework.

\subsection{Why Score Imitation Is the Wrong Default Goal}
\label{sec:why-imitation-wrong}

Many ``AI reviewer'' projects optimize agreement with human reviews or acceptance decisions \citep{zhou-etal-2024-llm,thelwall2025predictive,bao2021acceptance}.
Imitation can be a reasonable target for writing assistance, where we want outputs to match human style.
It is not a good default target for evaluating a fixed scientific claim.

A human score comes from the same bandwidth-limited system we aim to improve.
When $\vpressure$ is large and $\snr$ is large, Theorem~\ref{thm:tc} implies that truth-coupling $\tc$ is low.
In that regime, scores are weak signals of $\Truth$ by construction.
If we train or prompt an AI reviewer to match those scores, we teach it to reproduce the same proxy-weighted rule, including its blind spots. 
Even training on a historical ``high-SNR'' period can at best reduce $r$; it does not increase verification frequency $q$, and at scale imitation tools tend to raise $R$ and $\gamma$, lowering $q$ and reintroducing proxy-sovereign dynamics.

Imitation-first deployment can also worsen incentives.
Tools can increase the effective claim rate $R$ by making proxy-aligned submissions cheaper to produce, and increase $\gamma$ in Theorem~\ref{thm:effort} by making proxy criteria easier to optimize.
Both changes raise verification pressure and move the system toward proxy sovereignty.

Our claim is not ``automation instead of reviewers'': human reviews still contain useful, work-specific signal, and verification-first AI should complement humans by producing auditable evidence for them to weigh \citep{dutta2025problem}.

\subsection{The Objective: Reducing Verification Pressure}
\label{sec:objective}

Our theoretical analysis implies that we cannot simply optimize for truth-coupling $\tc$ as a static metric, as it can be a deceptive lagging indicator without capture structurally the deep fragile.
Instead, we must target the \emph{drivers} of the potential collapse.
Theorem~\ref{thm:effort} establishes that stability depends on keeping verification pressure $\vpressure = R\Ce/\Be$ low enough to satisfy the incentive constraint.
This transforms the design goal from "better prediction" to "capacity engineering," pointing to three concrete operational levers, reducing $R$ (Gatekeeping), reducing $\Ce$ (Auditing), or increasing $\Be$ (Scaling).

Unlike mimetic methods that try to fit $S$ to $T$ directly, verification-first AI targets these levers.
We formally state this design goal as an \emph{evidence acquisition problem}.
Let $\mathcal{D}$ be the distribution over submissions with observable proxy $\Proxy$ and latent truth $\Truth$.
The AI policy is not a scoring function, but an \emph{auditing policy} $\pi$ that selects a verification action $\Verify \sim \pi(\cdot\mid \Proxy)$ to reveal evidence.
The final score $\Score$ is then derived from this evidence.

We seek policies that maximize the revelation of truth within the community's bandwidth limits.
\begin{tcolorbox}[
    colback=green!5!white,
    colframe=green!60!black,
    title=Verification-First AI Peer Review,
    colbacktitle=green!20!white,
    coltitle=black,
    sharp corners=south,
    fonttitle=\small\bfseries,
    boxrule=0.2mm,
    boxsep=1mm,
    left=2mm,
    right=2mm
]
\begin{equation}
\label{eq:objective_pretty}
\max_{\pi}\;\; \Corr(\Score_\pi, \Truth) \;-\;\lambda\,
\E_{\pi}\!\bigl[\text{Cost}(\Verify)\bigr],
\end{equation}
\end{tcolorbox}
Here, $\lambda$ penalizes verification cost, and $\Score_\pi$ is the score derived from the evidence $\Verify$ under policy $\pi\in\Pi_{\text{verify}}$.
Crucially, we restrict $\Pi_{\text{verify}}$ to auditable evidence-producing actions, preventing the optimizer from exploiting cheap proxies.
This forces the tool to find the most efficient \emph{decisive} test that maintains truth-coupling.

\subsection{Design Guidelines from the ``$(q,r)$-Framework''}

\paragraph{Reducing $R$ without suppressing novelty.}
A naive response to high $R$ is to raise barriers to entry.
That can protect incumbents and can suppress high-risk ideas.
A verification-first alternative is to \emph{separate claim production from claim validation}.
For example, a venue can implement a staged pipeline where most submissions first receive a lightweight auditability check.
This check asks whether the core claims map to specific artifacts, namely code, data, proofs, or ablation plans.
Only submissions that pass this gate consume full reviewer bandwidth.

We can also reduce $R$ by changing what counts as a ``claim.''
If a paper makes ten loosely connected claims, reviewers cannot verify any of them well.
A venue can ask authors to declare one primary claim and to attach a verification plan for that claim.
This reframes evaluation as claim-evidence pairing rather than as an unconstrained narrative.

\paragraph{Reducing $\Ce$ by producing higher-fidelity evidence.}
Equation~\eqref{eq:Ceff} suggests that $\Ce$ drops when $\kappa$ increases.
AI can help by turning vague claims into explicit, testable units and by generating minimal tests with high diagnostic value.
We highlight four families of verification artifacts that AI can scale.

\emph{Claim-evidence maps.}
Given a manuscript, an assistant can extract the main claims and link each claim to the exact evidence that supports it.
This helps reviewers allocate attention and reduces time spent searching.

\emph{Smoke tests.}
Given code and data, an assistant can run a small set of checks that catch common failure modes, such as non-determinism, data leakage, and baseline misconfiguration.
These checks do not prove correctness, but they increase $\kappa$ by eliminating weak forms of evidence.

\emph{Variance audits.}
If a result depends on random seeds, an assistant can automatically run repeated trials, report dispersion, and flag results that rely on a small number of favorable runs.
This directly targets the instability component of $\Delta$.

\emph{Stress tests.}
For empirical claims, an assistant can generate perturbations, distribution shifts, or adversarial subsets that test whether a reported gain is robust.
For theoretical claims with implementations, an assistant can search for counterexamples within a declared scope.

Each family reduces $\Ce$ because it increases the amount of decisive information per unit reviewer effort.
These artifacts also reduce $\Var(\Delta)$, because they remove degrees of freedom for proxy gaming.

\paragraph{Increasing $\Be$ by compressing expert attention.}
The bandwidth $\Be$ is not just ``more reviewers.''
It is also the ability to compress expert attention into reusable checks.
AI can increase $\Be$ by standardizing and automating verification artifacts, by routing papers to reviewers with matching expertise, and by generating concise summaries of evidence that reviewers can trust.
A key requirement is that these summaries remain grounded in verifiable artifacts.
Summarizing a paper-like narrative does not increase $\Be$ in our sense.
Summarizing verified evidence does.

\paragraph{Flattening proxy gradients to resist Goodhart drift.}
Even with better verification, proxy signals will remain.
We should therefore reduce the incentive to over-optimize them.
Theorem~\ref{thm:effort} suggests one mechanism: reduce the marginal score benefit of proxy effort in the non-verification mode.
Mechanism design can help.

A probabilistic audit mechanism serves as a critical deterrent.
By subjecting a random subset of accepted papers to rigorous, high-cost verification (e.g., code reproduction), we introduce a \emph{negative expected utility} for fragile or fabricated results, even if the audit probability is low.
Crucially, this creates an asymmetric risk profile: if authors know that acceptance carries a probability of a high-fidelity audit, then proxy gains that cannot survive scrutiny become liabilities.


\subsection{Case Study: The ``SOTA'' Optimizer Mirage}
\label{sec:case_study}


\begin{figure}[t]
\centering
\begin{tcolorbox}[
    width=0.98\linewidth,
    boxsep=3pt,
    left=4pt, right=4pt,
    top=3pt, bottom=3pt,
    colback=blue!3!white,      
    colframe=blue!40!black,    
    boxrule=0.6pt,
    arc=1.5mm,
    title={A paper reports a +0.5\% gain on ImageNet. Real?},
    colbacktitle=blue!10!white,
    coltitle=black,
    fonttitle=\small\bfseries,
    sharp corners=south,
    standard jigsaw,
    opacityback=0.9
]
    \small
    \setlength{\parskip}{0.3em}

    \textbf{1. The imitation path (accelerating Universe B):}
    The AI acts as a \textit{simulator}.
    It scans the text, recognizes the ``SOTA'' claim, and generates a plausible-sounding review praising the ``novelty'' while nitpicking formatting or missing citations.

    \par\noindent\hrulefill\par 

    \textbf{2. The verification-first path (restoring Universe A):}
    The AI acts as an \textit{adversarial auditor}.
    Instead of producing an opinion from the narrative, it extracts the core claim and generates a \textbf{stress test} to verify it.
    It automatically runs a seed sweep and checks sensitivity to hyperparameters.
\end{tcolorbox}

\vspace{-0.2cm}
\caption{Imitation AI automates the \textit{appearance} of peer review, optimizing for style match. Verification-first AI automates the \textit{labor} of scientific scrutiny, optimizing for truth discovery.}
\label{fig:case}
\vspace{-0.5cm}
\end{figure}

Figure~\ref{fig:case} reveals a critical \emph{paradox of efficiency}.
Superficially, imitation-first AI appears to boost productivity by generating text instantly; in reality, it acts as an \emph{inflationary engine}.
By commoditizing the production of ``plausible" reviews, it inflates the noise floor ($R$) without expanding the actual verification bandwidth ($\Be$).
This creates a ``Review Theater" where agents mimic the rituals of evaluation without performing the work.
In contrast, verification-first AI introduces \emph{salutary friction}.
It does not merely ``check" the paper; it \emph{collapses the state space} of selective reporting ($\Var(\Delta)$).
By automating the ``dirty work" of code execution and sensitivity analysis, it structurally removes the degrees of freedom that authors use to mine for significance.
While this approach feels more ``expensive" per paper (requiring compute), it is the only mathematical path to flip the inequality in Theorem~\ref{thm:effort} and make truth-seeking ($e^* > 0$) the rational strategy again.

\subsection{Why AI Plausibly Can Help Verify Research Result}

Automatic review generators can be largely insensitive to faulty reasoning under a controlled counterfactual evaluation, indicating that text-level ``review mimicry'' fails precisely on core reviewing skills \citep{dycke2025automatic}.
This motivates verification-first designs that target \emph{structured} reasoning and checkable sub-tasks \citep{dycke2025stricta}.
Concretely, SciCoQA operationalizes verification as paper--code alignment auditing, turning a key class of claims into measurable, reusable QA artifacts \citep{scicoqa2026}.
Together, these works support a pragmatic thesis: AI is most credible for verification when it produces inspectable evidence trails. Appendix~\ref{sec:empirical} outlines concrete measurement strategies.



%% file: sections/07_alternative_views.tex
\section{Alternative Views}
\label{sec:altviews}

To engage credible disagreement, we describe several alternative views that we often hear in discussions of AI-assisted peer review.
We then explain why we still advocate a verification-first objective.

\subsection{View 1: Review Mimicry Is a Pragmatic Target}

A common view is that human review scores and texts are the best available labels at decision time.
Under this view, we should train LLMs to predict scores or to generate reviews that match past reviewer behavior because doing so can reduce workload and speed up decisions \citep{zhou-etal-2024-llm,thelwall2025predictive}.
This view is credible.
Scientific truth is not directly observable at review time, and reviewers already represent the community's current norms.

We agree that language models can help with \emph{communication} tasks, such as summarizing a submission, mapping claims to supporting evidence, or helping authors rewrite unclear parts.
However, score mimicry is a problematic \emph{optimization target}.
In our model, a score is an output of a bandwidth-constrained process with nonzero proxy noise, so truth-coupling can be low under high verification pressure $\vpressure$ even when reviewers act in good faith.
Optimizing an AI system to match scores therefore optimizes toward this low-coupling mixture; moreover, fully automatic review generation can miss core reviewing skills (e.g., detecting faulty reasoning), underscoring the need for human judgment \citep{dycke2025automatic}.
Mimicry can also increase the claim rate $R$ by making persuasive text cheaper and proxy heuristics easier to satisfy, which raises $\vpressure$ and can further reduce coupling.

We do not argue that LLM assistance is useless.
Rather, imitation should not be the primary objective for AI-assisted \emph{evaluation}.
Instead, AI should support \emph{verification-first} workflows by producing auditable evidence that humans can interpret and weigh.

\subsection{View 2: We Should Avoid AI in Review Entirely}

Another credible view is that we should not use LLMs in peer review at all.
Concerns include confidentiality, hidden biases, and the risk that automated tools create new forms of review theater \citep{akella2025prereview}.
This view is also credible because of high-stakes setting and costly failure.

We acknowledge these risks, but we argue that \emph{abstinence is a strategy for certain defeat}.
Our theoretical analysis reveals a harsh reality: the verification bottleneck is not static; it is widening exponentially due to the rising claim rate $R$.
As long as human bandwidth $\Be$ remains fixed, the collapse condition derived in Theorem 4.4 becomes a \emph{mathematical inevitability}, regardless of ethical intent.
Rejecting AI assistance does not preserve the sanctity of human review; it structurally condemns the system to be overwhelmed by noise, forcing a regression to Universe B.
Therefore, the deployment of AI is not optional but existential.
We must deploy AI to radically expand $\Be$---not to mimic opinions, but to execute the artifact-grounded audits (e.g., code execution, seed sweeping) that act as the only mathematical counterweight to the coming flood of cheap claims.

\subsection{View 3: The Main Problem Is Bias and Inconsistency, Not Verification}

A third view is that the main failure mode of peer review is social bias and inconsistency.
Under this view, AI should focus on calibrating scores, detecting reviewer bias, or standardizing rubrics.
We agree that bias and inconsistency are real.
We also agree that some proxy signals can encode status and network effects.

We nevertheless expect verification-first mechanisms to reduce bias: when truth-coupling is low, small social effects can dominate decisions, but decisive evidence leaves fewer degrees of freedom for subjective bias.
Thus our framework complements fairness-oriented work: bias reduction can be an additional objective, but not one pursued via score imitation.
Instead, we should improve evidence fidelity and reduce proxy noise, strengthening both epistemic reliability and procedural fairness.


%% file: sections/08_call_to_action.tex
\section{Call to Action}
\label{sec:call}

Our theoretical analysis implies concrete steps.
We emphasize feasible changes within existing conference and journal structures.

\subsection{What Venues Can Do}
\label{sec:venues}

Venues should treat verification bandwidth as scarce and allocate it deliberately.
A practical step is a staged review pipeline.
Stage~1 screens \emph{auditability}: reviewers and automated checks verify that the main claims map to concrete artifacts, scripts, data, or proofs.
Stage~2 concentrates deep verification on a smaller set of submissions, including artifact reruns and robustness checks.
This reduces the effective claim rate $R$ that consumes full bandwidth.

Venues should also commit to random post-acceptance audits.
Even a small audit rate weakens incentives to optimize fragile proxies.
Audits can be lightweight but decisive, such as rerunning code, checking baselines, or validating a theorem.
Audit outcomes can be published as additive addenda, so the process stays corrective rather than punitive.

\subsection{What Tool Builders Can Do}
\label{sec:tool-builders}

Tool builders should prioritize evidence-grounded modules over review text generation.
Examples include claim--evidence maps, reproducibility smoke tests, variance audits, and stress tests.
Outputs should be inspectable and reusable, namely commands, logs, and minimal evidence bundles.
We should evaluate tools by whether they reduce $\Ce$, increase $\Be$, or reduce $\snr$ in downstream decisions.

Tool builders should also support shared, open audit suites.
Community-maintained diagnostic tests let venues reuse results and compress expert effort into standardized checks, increasing effective bandwidth.

\subsection{What Funders and Institutions Can Do}
\label{sec:funders}

Verification-first systems require infrastructure.
Funders can support shared testbeds, reproducibility tooling, and long-horizon evaluation datasets.
Institutions can recognize verification work as a first-class contribution, including replications, negative results, and engineering that makes evidence cheaper to check.


%% file: sections/10_discussion.tex
\section{Conclusion}
\label{sec:conclusions}

AI can either amplify proxy sovereignty or help restore truth sovereignty.
Our results suggest that the core  challenge is 
not to generate better review prose.
It is 
to design evidence, incentives, and bandwidth allocation so that ``what gets rewarded'' stays coupled to ``what is true.''

\section*{Acknowledgement}
We thank Dr.\ Nihar Shah for insightful discussions and for early, sustained advocacy of verification-oriented uses of AI in peer review, and Nils Dycke and Subhabrata Dutta for valuable comments that helped clarify assumptions and sharpen the presentation.


%% file: appendices/appendix_overview.tex
\section{Appendix Overview}
\label{app:overview}

This appendix collects proofs and extensions that support the main text.

Appendix~\ref{sec:related} summarizes the related work to this paper.
Appendix~\ref{app:proofs} provides full proofs for Theorems~\ref{thm:tc} and \ref{thm:effort} and Corollary~\ref{cor:budget}.
Appendix~\ref{app:goodhart-amplification} provides a stylized derivation showing how gaming variance can grow rapidly with verification pressure when actors select the best outcome from many proxy-optimization attempts.
Appendix~\ref{app:saturation} discusses practical ways to measure signal shrinkage and the noise-to-signal ratio $\snr$ in real venues.
Appendix~\ref{app:citations-proxy} formalizes how low truth–coupling regimes can arise and shows large cross-field disparities in the expected scientific value per citation can follow.
Appendix~\ref{sec:empirical} outlines how we can measure verification pressure and signal shrinkage in practice.
Appendix~\ref{app:cases} gives short case studies that connect the model to several mature engineering domains.
Appendix~\ref{sec:call} proposes concrete steps that venues, tool builders, and funders can take. 
Appendix~\ref{sec:limitations} discussed the limitations of this paper.

%% file: appendices/03_related_work.tex
\section{Related Work}
\label{sec:related}

\subsection{Metrics, Incentives, and Goodhart Drift}

A central concern in our paper is the gap between a latent objective and a measurable proxy.
Goodhart's law states that when a measure becomes a target, it can stop being a good measure \citep{goodhart1975problems}.
Campbell's law articulates a related effect in social systems under pressure \citep{campbell1979assessing}.
Recent work has clarified that Goodhart effects arise through multiple mechanisms, including selection, causal intervention, and adversarial behavior \citep{manheim2019categorizing}.
The result is a feedback loop that actively narrows scientific horizons: recent empirical work confirms that while AI boosts individual productivity, it leads to a collective ``flattening'' of discovery, where outputs cluster around existing canons rather than exploring risky novelty \cite{hao2026artificial}.

Our model also draws from the literature on metric-driven organizations.
In the ``audit society'' view, modern institutions increase reliance on formal indicators because indicators are legible and defensible under accountability constraints \citep{power1997audit,strathern2000audit}.
This shift can improve consistency but can also displace the underlying purpose when indicators become the primary object of optimization \citep{muller2018tyranny}.

Our framework unifies these classical sociological insights into a formal control model.
While \cite{smaldino2016natural} famously characterized the proliferation of bad science as a ``natural selection'' process driven by incentives, and \cite{ioannidis2005most} identified the statistical inevitability of false findings, these works largely treated the verification constraint (human bandwidth) as a fixed constant.
Our work differs by introducing AI as a dynamic variable. We argue that the ``natural selection'' toward Universe B is not absolute destiny, but a function of verification cost, which modern AI has the potential to fundamentally alter. 

\subsection{Limited Attention and Verification Bottlenecks}

We treat peer review as an attention allocation process under scarcity.
This view aligns with rational inattention, which models decision makers as optimizing under information processing constraints \citep{sims2003rational}.
It also aligns with multi-task incentive theory, which predicts that when agents face many tasks but are evaluated on a subset, effort shifts toward what is measured \citep{holmstrom1991multitask}.

\subsection{Peer Review Datasets and Automation}

Several datasets and models study peer review as text and as a prediction problem.
PeerRead provides a public corpus of reviews and accept decisions across major venues \citep{kang-etal-2018-dataset}.
Prediction systems include interpretable decision rules for acceptance forecasting \citep{bao2021acceptance}.

The recent wave of LLMs revived interest in automated reviewing and, more broadly, AI-assisted scholarly critique \citep{zou2024chatgpt}.
Empirical studies have evaluated LLMs as review generators and compared their outputs with expert reviews \citep{zhou-etal-2024-llm,shin2025automatic}.
Complementary work investigates how LLM-assisted feedback can enrich critique beyond simple novelty screening, for example by producing more specific, actionable feedback \citep{afzal2025beyond}.
Other studies test whether LLM-derived scores correlate with conference or journal outcomes \citep{thelwall2025predictive}. 
Finally, structured approaches aim to formalize reviewer reasoning workflows (e.g., via causal or structured reasoning representations), which can guide where and how LLM-based interventions should assist humans \citep{dycke2025stricta}.

At the same time, LLMs already affect the peer review ecosystem as writers and editors.
For example, a large-scale analysis estimated measurable LLM modification in review text after the release of ChatGPT \citep{liang2024monitoring}.

Our stance differs from most ``LLM-as-reviewer'' work.
We do not treat agreement with human review text or scores as the primary objective.
We treat it as a potentially misleading target when the human process itself operates in a proxy-sovereign regime.
We align more closely with work warning against over-automation of peer review \citep{akella2025prereview} and with evidence that automated assistants can be gamed when judging proxy compliance \citep{goldberg2024checklist}.

\subsection{Psychohistory and Sociophysics}
Our approach parallels the ``Psychohistory'' concept envisioned by Asimov, which posits that while individual trajectories are unpredictable, the statistical behavior of massive populations follows deterministic laws \cite{asimov1951foundation}. In the real world, this intuition is formalized in Cliodynamics and Sociophysics. Turchin's theory of "elite overproduction" \citep{turchin2016ages} offers a striking analog to our model of claim inflation ($R$): when the production of claims outstrips the community's verification capacity ($\Be$), the system enters a phase of instability similar to secular cycles in history. Furthermore, Tainter's framework on the collapse of complex societies \citep{tainter1988collapse} explains the phenomenon of defensive proxy optimization: as a field matures, the marginal return on "complexity" (e.g., denser appendices, heavier tables) diminishes, yet the system demands more of it to maintain the illusion of progress, accelerating the drift toward Universe B. Our work translates these macro-historical dynamics into the micro-incentive structure of peer review.

%% file: appendices/appendix_proofs.tex
\section{Proofs}
\label{app:proofs}

\subsection{Proof of Theorem~\ref{thm:tc}}

\begin{proof}
We start from the score model in \eqref{eq:score},
\[
\Score = \Truth + (1-q)\Delta.
\]
Assumption~\ref{assump:mixture} states that $\Delta$ has mean zero and is independent of $\Truth$.

We first compute the covariance between $\Score$ and $\Truth$.
Using linearity of covariance, we have
\begin{align*}
\Cov(\Score,\Truth)
&= \Cov\!\left(\Truth + (1-q)\Delta,\,\Truth\right) \\
&= \Cov(\Truth,\Truth) + (1-q)\Cov(\Delta,\Truth).
\end{align*}
Independence implies $\Cov(\Delta,\Truth)=0$, so
\[
\Cov(\Score,\Truth) = \Var(\Truth).
\]

We next compute the variance of $\Score$.
Using $\Var(X+Y)=\Var(X)+\Var(Y)+2\Cov(X,Y)$, we obtain
\begin{align*}
\Var(\Score)
&= \Var\!\left(\Truth + (1-q)\Delta\right) \\
&= \Var(\Truth) + (1-q)^2\Var(\Delta) + 2(1-q)\Cov(\Truth,\Delta).
\end{align*}
Independence again implies $\Cov(\Truth,\Delta)=0$, so
\[
\Var(\Score) = \Var(\Truth) + (1-q)^2\Var(\Delta).
\]

By the definition of correlation,
\[
\Corr(\Score,\Truth)
= \frac{\Cov(\Score,\Truth)}{\sqrt{\Var(\Score)\Var(\Truth)}}
= \frac{\Var(\Truth)}{\sqrt{\Var(\Truth)\left(\Var(\Truth)+(1-q)^2\Var(\Delta)\right)}}.
\]
We cancel a factor of $\sqrt{\Var(\Truth)}$ and obtain
\[
\tc = \Corr(\Score,\Truth)
= \left(1 + (1-q)^2\frac{\Var(\Delta)}{\Var(\Truth)}\right)^{-1/2}.
\]
Recalling the definition $\snr=\Var(\Delta)/\Var(\Truth)$ from \eqref{eq:snr} completes the proof.
\end{proof}

\subsection{Proof of Corollary~\ref{cor:budget}}

\begin{proof}
Theorem~\ref{thm:tc} states that
\[
\tc = \left(1 + (1-q)^2\snr\right)^{-1/2}.
\]
We want to characterize when $\tc \ge \tc_{\min}$.

Because all terms are nonnegative and $\tc_{\min}\in(0,1)$, we can square both sides without changing the inequality direction:
\[
\left(1 + (1-q)^2\snr\right)^{-1} \ge \tc_{\min}^2.
\]
We invert both sides.
Inversion reverses inequalities for positive numbers, so we obtain
\[
1 + (1-q)^2\snr \le \tc_{\min}^{-2}.
\]
We subtract one from both sides and divide by $\snr>0$:
\[
(1-q)^2 \le \frac{\tc_{\min}^{-2}-1}{\snr}.
\]
The right-hand side is nonnegative because $\tc_{\min}\in(0,1)$.
We take square roots on both sides and obtain
\[
1-q \le \sqrt{\frac{\tc_{\min}^{-2}-1}{\snr}}.
\]
Finally, we rearrange to conclude that
\[
q \ge 1 - \sqrt{\frac{\tc_{\min}^{-2}-1}{\snr}},
\]
which matches \eqref{eq:qmin}.
\end{proof}

\subsection{Proof of Theorem~\ref{thm:effort}}

\begin{proof}
We consider the optimization problem
\[
\max_{e\in[0,1]} U(e)
\quad \text{where} \quad
U(e)=f(e)+(1-q)\gamma(1-e).
\]
Assumption~\ref{assump:effort} states that $f$ is increasing, concave, and differentiable.
The function $e \mapsto (1-q)\gamma(1-e)$ is linear in $e$.
Therefore $U$ is concave on $[0,1]$, as a sum of a concave function and a linear function.
A concave function on a compact interval attains its maximum, and any point that satisfies the first-order optimality condition and lies in the interior is a global maximizer.

We compute the derivative of $U$:
\[
U'(e) = f'(e) - (1-q)\gamma.
\]
If an optimal solution $e^\star$ lies strictly inside $(0,1)$, then the necessary optimality condition is $U'(e^\star)=0$.
This yields the first-order condition
\[
f'(e^\star) = (1-q)\gamma,
\]
which is \eqref{eq:foc}.

We now characterize when the optimum occurs at the boundary $e=0$.
Because $U$ is concave, the derivative $U'(e)$ is nonincreasing in $e$.
If $U'(0)\le 0$, then $U'(e)\le 0$ for all $e\in[0,1]$, and $U$ is nonincreasing.
In that case, $e^\star=0$ is optimal.
We compute
\[
U'(0) = f'(0) - (1-q)\gamma.
\]
The condition $U'(0)\le 0$ is equivalent to $(1-q)\gamma \ge f'(0)$.
This proves the stated collapse condition.
\end{proof}

%% file: appendices/appendix_goodhart_amplification.tex
\section{A Stylized Goodhart Amplification Mechanism}
\label{app:goodhart-amplification}

The main text treats $\mathrm{Var}(\Delta)$ as a primitive. In many fields, however, $\mathrm{Var}(\Delta)$ is endogenous:
when incentives reward proxy performance, researchers can expand the \emph{search budget} devoted to proxy optimization.
This appendix sketches a minimal mechanism---\emph{best-of-$K$ selection}---showing how proxy-selection noise can grow quickly
with verification pressure $\Lambda$.

\subsection{Best-of-$K$ Selection Under Proxy Optimization}\label{app:goodhart-best-of-k}

Consider a proxy objective (e.g., benchmark score) that can be improved by repeated attempts such as hyperparameter sweeps,
prompt sweeps, or variant sweeps. Let $X\ge 0$ denote the proxy gain from a single attempt, and let $X_1,\dots,X_K$ be i.i.d.\ copies of $X$.
If only the best outcome is reported, the reported proxy gain is
\[
G_K := \max\{X_1,\dots,X_K\}.
\]
If the best-of-$K$ gain does not reliably survive independent reruns, a simple stylized contribution to the proxy error is the
\emph{centered selection term}
\[
\Delta_{\mathrm{sel}} := G_K - \mathbb{E}[G_K],
\qquad\text{so that}\qquad
\mathrm{Var}(\Delta) \;\ge\; \mathrm{Var}(\Delta_{\mathrm{sel}}) = \mathrm{Var}(G_K).
\]
The key driver is how $K$ scales. Verification pressure creates competitive pressure, and AI reduces the marginal cost per attempt;
both effects can increase $K$. We therefore treat $K$ as a (weakly) increasing function of $\Lambda$.

A standard order-statistics identity gives, for any CDF $F$ of $X$,
\[
\mathbb{P}(G_K \le x) = \mathbb{P}(X\le x)^K = F(x)^K.
\]
Thus any typical scale of $G_K$ (e.g., quantiles) increases with $K$.

\subsection{Heavy-Tailed Proxy Gains Yield Fast Growth}\label{app:goodhart-heavy-tail}

Growth can be especially rapid when proxy gains are heavy-tailed: most tweaks yield tiny gains, but rare tweaks yield very large gains.
We use Pareto tails as a transparent example.

\begin{assumption}[Pareto proxy gain]\label{assump:pareto-proxy-gain}
The proxy gain $X$ follows a Pareto distribution with parameters $x_{\min}>0$ and $\alpha>2$:
\[
\mathbb{P}(X > x)=\left(\frac{x_{\min}}{x}\right)^{\alpha},\qquad x\ge x_{\min}.
\]
The condition $\alpha>2$ ensures $X$ has finite variance.
\end{assumption}

\begin{proposition}[Typical scale of best-of-$K$ grows polynomially]\label{prop:best-of-k-typical}
Under Assumption~\ref{assump:pareto-proxy-gain}, the median of $G_K$ is
\[
\operatorname{median}(G_K)
= x_{\min}\bigl(1-2^{-1/K}\bigr)^{-1/\alpha}
\;\approx\;
 x_{\min}\Bigl(\tfrac{K}{\ln 2}\Bigr)^{1/\alpha}
\quad\text{for large $K$}.
\]
Consequently, the typical scale of best-of-$K$ outcomes grows on the order of $K^{1/\alpha}$.
When $\alpha>2$, this implies that the dispersion of $G_K$ (and thus the selection component of $\mathrm{Var}(\Delta)$) increases at least polynomially with $K$.
\end{proposition}

\noindent\emph{Derivation (standard order-statistics).}
For any $x\ge x_{\min}$, $\mathbb{P}(G_K\le x)=F(x)^K$.
The median $m_K$ solves $F(m_K)^K=\tfrac12$, hence $F(m_K)=2^{-1/K}$.
For Pareto, $F(x)=1-(x_{\min}/x)^{\alpha}$, so
$1-(x_{\min}/m_K)^{\alpha}=2^{-1/K}$, yielding
$m_K=x_{\min}(1-2^{-1/K})^{-1/\alpha}$.
For large $K$, $2^{-1/K}=\exp(- (\ln 2)/K)\approx 1-(\ln 2)/K$, which gives the stated approximation.

\subsection{A Route to Exponential Amplification in Verification Pressure}\label{app:goodhart-exponential}

Proposition~\ref{prop:best-of-k-typical} gives polynomial growth in $K$. If $K$ itself grows rapidly with verification pressure,
then proxy-selection noise can grow rapidly with $\Lambda$.

\begin{corollary}[A route to exponential growth in proxy-selection noise]\label{cor:exp-growth}
Assume $K(\Lambda)=K_0\,\exp\bigl(\beta(\Lambda-1)_+\bigr)$ for constants $K_0\ge 1$ and $\beta>0$, where $(x)_+=\max\{x,0\}$.
Under Assumption~\ref{assump:pareto-proxy-gain}, the median (and typical scale) of $G_{K(\Lambda)}$ grows exponentially in $\Lambda$:
\[
\operatorname{median}\bigl(G_{K(\Lambda)}\bigr)
\;\approx\;
 x_{\min}\Bigl(\tfrac{K_0}{\ln 2}\Bigr)^{1/\alpha}\exp\Bigl(\tfrac{\beta}{\alpha}(\Lambda-1)_+\Bigr).
\]
Therefore, the contribution of best-of-$K$ selection to $\mathrm{Var}(\Delta)$ can increase exponentially with verification pressure.
\end{corollary}

This appendix is intentionally stylized. It does \emph{not} claim that all fields exhibit Pareto proxy gains or exponential growth in $K$.
The point is qualitative: once (i) the number of proxy-optimization attempts $K$ becomes endogenous and increases with pressure, and
(ii) selection picks the best outcome, proxy-selection noise can grow quickly.
This pushes up $r=\mathrm{Var}(\Delta)/\mathrm{Var}(T)$ and can accelerate drift toward proxy sovereignty.


%% file: appendices/appendix_measuring_saturation.tex
\section{Measuring Signal Shrinkage in Practice}
\label{app:saturation}

The main text uses $\Var(\Truth)$ as a proxy for how readable the true signal is in a given venue-year.
This appendix discusses measurement strategies and practical proxies.
Our goal is not to define ``truth'' perfectly.
Our goal is to approximate whether a community faces a low signal-to-noise regime.

\subsection{Benchmark Headroom as A Saturation Proxy}

In many machine learning subfields, progress is reported on fixed benchmarks.
If a benchmark metric approaches a ceiling, then headroom shrinks.
This creates two effects.
First, true improvements become smaller.
Second, the same absolute amount of evaluation noise becomes more influential.

A simple headroom proxy is $1-M_t$, where $M_t$ is the best reported metric value in year $t$ on a bounded metric.
If headroom shrinks over time, we expect the variance of improvements $\Delta M_t = M_t - M_{t-1}$ to shrink.
This shrinkage is one empirical signal of declining $\Var(\Truth)$.

This proxy has limitations.
Benchmarks can shift, and reported best values can reflect selective reporting.
For this reason, headroom is most informative when combined with verification artifacts such as reruns or standardized evaluation suites.

\subsection{Variance Decomposition via Audits}

We can estimate $\snr=\Var(\Delta)/\Var(\Truth)$ without observing $\Truth$ directly by using audits as partial ground truth.
Suppose an audit protocol reruns a reported experiment with controlled randomness and standardized baselines.
Let $Y$ denote the originally reported gain and let $Y'$ denote the audited gain.
A large dispersion of $Y-Y'$ indicates large proxy error variance.

If we treat $Y'$ as a higher-fidelity measurement of the underlying effect, then the variance of $Y'$ across papers approximates $\Var(\Truth)$, while the variance of $Y-Y'$ approximates $\Var(\Delta)$.
This approach is imperfect, but it matches the design goal of verification-first review: we want evidence that supports variance decomposition.

\subsection{Long-Horizon Stability as A Truth Proxy}

A central challenge is that short-horizon venue decisions cannot directly observe long-horizon truth.
We can therefore use stability proxies.

One proxy is \emph{baseline durability}.
If a method remains competitive when later work introduces stronger baselines and broader sweeps, the original claim likely had higher truth content.

Another proxy is \emph{replication survival}.
If independent groups reproduce the claimed gain, the claim likely had higher truth content.
Community replication datasets can support this measurement.

A third proxy is \emph{mechanistic compression}.
In some fields, a contribution becomes durable when it compresses into a reusable concept, theorem, or primitive.
Citation patterns can partly capture this, although they remain social signals.

\subsection{Why Saturation Interacts With Verification Pressure}

Our main claim about saturation is that it amplifies the effect of verification pressure.
When $\Var(\Truth)$ shrinks, the same level of proxy noise produces a larger $\snr$.
In that regime, even honest incremental work becomes hard to distinguish.
This encourages communities to rely more on legible proxies.
It also encourages researchers to invest in proxy optimization because small, fragile gains can win selection if they look like progress.

This interaction helps explain why mature engineering domains can drift toward proxy sovereignty even when they retain strong technical cultures.
When true improvements become extremely fine-grained, the system needs more decisive verification per claim to maintain coupling.
If bandwidth does not scale, the system drifts.

%% file: appendices/appendix_H_citations_proxy.tex
\section{Citations as Proxy Signals and a Winner's Curse}
\label{app:citations-proxy}

We can use our framework to reason about citation counts. In many fields, citations act as a low-cost recognition signal. Citations can reflect genuine scientific value, but they can also reflect visibility, coordination effects, and social amplification. In this appendix we treat citations as another proxy score and we derive a simple \emph{winner's curse} effect: when truth-coupling is low, papers with extreme citation scores can have modest latent scientific value.

\subsection{A Stylized Citation Score Model}

We let $T$ denote latent scientific value, as in the main text. We define a centered real-valued citation score $C$. In practice, we can construct $C$ from citation counts by taking a log transform and then standardizing, so $C$ has mean zero and finite variance.

\begin{assumption}[Citation proxy decomposition]
\label{assump:citation-decomp}
We assume citations follow the same mixture structure as venue scores. Specifically,
\begin{equation}
\label{eq:citation-mixture}
C = T + (1-q)\Delta_C,
\end{equation}
where $q \in [0,1]$ is the effective verification rate from the main text. The random variable $\Delta_C$ captures citation-specific proxy effects. We assume $\mathbb{E}[\Delta_C]=0$ and $\Delta_C$ is independent of $T$.
\end{assumption}

We define the citation noise-to-signal ratio
\begin{equation}
\label{eq:rC-def}
r_C := \frac{\mathrm{Var}(\Delta_C)}{\mathrm{Var}(T)}.
\end{equation}

\begin{corollary}[Citation truth-coupling]
\label{thm:citation-coupling}
Under Assumption~\ref{assump:citation-decomp}, the correlation between citations and latent value equals
\begin{equation}
\label{eq:rhoC}
\rho_C := \mathrm{Corr}(C,T) = \Bigl(1 + (1-q)^2 r_C\Bigr)^{-1/2}.
\end{equation}
\end{corollary}
Corollary~\ref{thm:citation-coupling} is derived directly from Theorem~\ref{thm:tc}.

\subsection{A Winner's Curse for Highly Cited Papers}

Corollary~\ref{thm:citation-coupling} shows that citations can decouple from latent value when $q$ is small or $r_C$ is large. We now derive a sharper statement about selection on high citation scores. We add a Gaussian assumption only to obtain a closed-form conditional expectation.

\begin{assumption}[Gaussian proxy model]
\label{assump:citation-gaussian}
In addition to Assumption~\ref{assump:citation-decomp}, we assume $T$ and $\Delta_C$ are independent Gaussian random variables with mean zero and finite variances.
\end{assumption}

\begin{theorem}[Winner's curse for citation scores]
\label{thm:winners-curse-citations}
Under Assumptions~\ref{assump:citation-decomp} and \ref{assump:citation-gaussian}, for any $c \in \mathbb{R}$ we have
\begin{equation}
\label{eq:cond-mean-T-given-C}
\mathbb{E}[T\mid C=c] = \rho_C^2\,c.
\end{equation}
Moreover, the conditional variance equals
\begin{equation}
\label{eq:cond-var-T-given-C}
\mathrm{Var}(T\mid C=c) = \mathrm{Var}(T)\bigl(1-\rho_C^2\bigr),
\end{equation}
and the expected proxy inflation in the citation score equals
\begin{equation}
\label{eq:cond-inflation}
\mathbb{E}[C-T\mid C=c] = \bigl(1-\rho_C^2\bigr)c.
\end{equation}
\end{theorem}

\begin{proof}
We write $a:=1-q$ for brevity. By \eqref{eq:citation-mixture}, we have $C=T+a\Delta_C$. Since $T$ and $\Delta_C$ are independent Gaussian, the pair $(T,C)$ is jointly Gaussian.

We first compute $\mathrm{Cov}(T,C)$ and $\mathrm{Var}(C)$. Using linearity and independence,
\begin{align*}
\mathrm{Cov}(T,C)
&= \mathrm{Cov}\bigl(T,\,T+a\Delta_C\bigr)
= \mathrm{Var}(T) + a\,\mathrm{Cov}(T,\Delta_C)
= \mathrm{Var}(T),
\end{align*}
and
\begin{align*}
\mathrm{Var}(C)
&= \mathrm{Var}\bigl(T+a\Delta_C\bigr)
= \mathrm{Var}(T) + a^2\mathrm{Var}(\Delta_C) + 2a\,\mathrm{Cov}(T,\Delta_C)
= \mathrm{Var}(T) + a^2\mathrm{Var}(\Delta_C).
\end{align*}

Let
\begin{align*}
\beta := \frac{\mathrm{Cov}(T,C)}{\mathrm{Var}(C)} = \frac{\mathrm{Var}(T)}{\mathrm{Var}(C)}.
\end{align*}
We define the residual random variable $U := T-\beta C$. We compute its covariance with $C$:
\begin{align*}
\mathrm{Cov}(U,C)
&= \mathrm{Cov}(T-\beta C,\,C)
= \mathrm{Cov}(T,C) - \beta\,\mathrm{Var}(C)
= 0.
\end{align*}
Since $(U,C)$ is jointly Gaussian and uncorrelated, $U$ is independent of $C$.\footnote{For jointly Gaussian random variables, zero covariance implies independence. This is a standard fact about multivariate normal distributions.} This implies
\begin{align*}
\mathbb{E}[T\mid C]
&= \mathbb{E}[\beta C + U\mid C]
= \beta C + \mathbb{E}[U\mid C]
= \beta C + \mathbb{E}[U]
= \beta C,
\end{align*}
where we use $\mathbb{E}[U]=0$ because $\mathbb{E}[T]=\mathbb{E}[C]=0$.
Conditioning on $C=c$ yields $\mathbb{E}[T\mid C=c]=\beta c$.

We now express $\beta$ in terms of $\rho_C$. By definition,
\begin{align*}
\rho_C^2
= \frac{\mathrm{Cov}(T,C)^2}{\mathrm{Var}(T)\,\mathrm{Var}(C)}
= \frac{\mathrm{Var}(T)^2}{\mathrm{Var}(T)\,\mathrm{Var}(C)}
= \frac{\mathrm{Var}(T)}{\mathrm{Var}(C)}
= \beta.
\end{align*}
Substituting $\beta=\rho_C^2$ proves \eqref{eq:cond-mean-T-given-C}.

Next we compute the conditional variance. Since $T=\beta C + U$ with $U$ independent of $C$, we have
\begin{align*}
\mathrm{Var}(T\mid C)
= \mathrm{Var}(U\mid C)
= \mathrm{Var}(U)
= \mathrm{Var}(T-\beta C).
\end{align*}
Expanding this variance and using $\beta=\mathrm{Cov}(T,C)/\mathrm{Var}(C)$, we obtain
\begin{align*}
\mathrm{Var}(U)
&= \mathrm{Var}(T) + \beta^2\mathrm{Var}(C) - 2\beta\,\mathrm{Cov}(T,C)\\
&= \mathrm{Var}(T) + \frac{\mathrm{Cov}(T,C)^2}{\mathrm{Var}(C)} - 2\frac{\mathrm{Cov}(T,C)^2}{\mathrm{Var}(C)}\\
&= \mathrm{Var}(T) - \frac{\mathrm{Cov}(T,C)^2}{\mathrm{Var}(C)}
= \mathrm{Var}(T)\bigl(1-\rho_C^2\bigr),
\end{align*}
which gives \eqref{eq:cond-var-T-given-C}.

Finally, we use $C-T = a\Delta_C$ from \eqref{eq:citation-mixture}. Taking conditional expectations and using \eqref{eq:cond-mean-T-given-C} yields
\begin{align*}
\mathbb{E}[C-T\mid C=c] = c - \mathbb{E}[T\mid C=c] = \bigl(1-\rho_C^2\bigr)c,
\end{align*}
which proves \eqref{eq:cond-inflation}.
\end{proof}

Equation~\eqref{eq:cond-inflation} makes the winner's curse explicit. When $\rho_C$ is small, an extreme citation score $c$ can have a large expected inflation component $\mathbb{E}[C-T\mid C=c]$. At the same time, \eqref{eq:cond-var-T-given-C} shows that uncertainty about $T$ remains large when $\rho_C$ is small, even after we condition on $C=c$.

\begin{corollary}[Selection on the most cited paper]
\label{cor:top-cited}
Consider $N$ independent papers with i.i.d. pairs $(T_i,C_i)$ that satisfy Assumptions~\ref{assump:citation-decomp} and \ref{assump:citation-gaussian}. Let $i^*\in\arg\max_{i\in\{1,\dots,N\}} C_i$ and let $C_{\max}:=C_{i^*}$. Then for any $c\in\mathbb{R}$,
\begin{equation}
\label{eq:top-cited-mean}
\mathbb{E}[T_{i^*}\mid C_{\max}=c] = \rho_C^2\,c.
\end{equation}
\end{corollary}

\begin{proof}
By definition, the event $\{C_{\max}=c\}$ implies $\{C_{i^*}=c\}$ and $\{C_j\le c\ \text{for all } j\ne i^*\}$. Since the papers are independent, the random pair $(T_{i^*},C_{i^*})$ is independent of the collection $\{C_j\}_{j\ne i^*}$. Therefore, conditioning on $\{C_j\le c\ \text{for all } j\ne i^*\}$ does not change the conditional distribution of $T_{i^*}$ given $C_{i^*}=c$. We thus have
\begin{align*}
\mathbb{E}[T_{i^*}\mid C_{\max}=c]
= \mathbb{E}[T_{i^*}\mid C_{i^*}=c]
= \rho_C^2\,c,
\end{align*}
where the last equality follows from Theorem~\ref{thm:winners-curse-citations}.
\end{proof}

Corollary~\ref{cor:top-cited} shows that selecting papers by extreme citation scores does not remove proxy inflation. Equation~\eqref{eq:top-cited-mean} implies that the expected latent value of the most cited paper scales like $\rho_C^2$ times its citation score. When $\rho_C<1$, the selected paper also satisfies $\mathbb{E}[C_{\max}-T_{i^*}\mid C_{\max}=c] = (1-\rho_C^2)c$ by \eqref{eq:cond-inflation}. This matches the qualitative phenomenon that some highly cited papers need not have correspondingly high latent scientific value.

\subsection{Population-Normalized Citation Scores}

To compare citation signals across fields, we need a normalization that accounts for field size and citation ``volume.'' We keep the model as simple as possible.

\begin{definition}[Population-normalized citation score]
\label{def:pop-normalized-citations}
Consider a field $f$. We let $Y_f$ denote the raw citation count of a paper measured over a fixed horizon. We let $\mu_f:=\mathbb{E}[Y_f]$ denote the field baseline mean over the same horizon. We let $s_f>0$ denote a scaling factor that captures population effects, such as the number of active citers or the number of papers that can generate citations in the horizon. We define the population-normalized citation score
\begin{equation}
\label{eq:pop-normalized-score}
C_f := \frac{Y_f-\mu_f}{s_f}.
\end{equation}
\end{definition}

We can now apply the same proxy decomposition to the normalized score.

\begin{assumption}[Field-specific citation proxy model]
\label{assump:field-citation-model}
For each field $f$, the normalized score $C_f$ satisfies
\begin{equation}
\label{eq:field-citation-mixture}
C_f = T_f + (1-q_f)\Delta_{C,f},
\end{equation}
where $T_f$ denotes latent scientific value in field $f$. We assume $\mathbb{E}[\Delta_{C,f}]=0$ and $\Delta_{C,f}$ is independent of $T_f$.
\end{assumption}

We define $r_{C,f}:=\mathrm{Var}(\Delta_{C,f})/\mathrm{Var}(T_f)$ and $\rho_{C,f}:=\mathrm{Corr}(C_f,T_f)$. By Corollary~\ref{thm:citation-coupling}, we have
\begin{align*}
\rho_{C,f} = \bigl(1+(1-q_f)^2 r_{C,f}\bigr)^{-1/2}.
\end{align*}

\subsection{Expected Scientific Value Given Raw Citations}

We can translate the theories back to raw citation units.

\begin{assumption}[Field-specific Gaussian proxy model]
\label{assump:field-citation-gaussian}
In addition to Assumption~\ref{assump:field-citation-model}, we assume $T_f$ and $\Delta_{C,f}$ are independent Gaussian random variables with mean zero and finite variances.
\end{assumption}

\begin{theorem}[Conditional value in raw citation units]
\label{thm:raw-citation-conditional-value}
Under Assumptions~\ref{assump:field-citation-model} and \ref{assump:field-citation-gaussian}, for any field $f$ and any $y\in\mathbb{R}$ we have
\begin{equation}
\label{eq:ET-given-raw}
\mathbb{E}[T_f\mid Y_f=y] = \rho_{C,f}^2\,\frac{y-\mu_f}{s_f}.
\end{equation}
\end{theorem}

\begin{proof}
We condition on $Y_f=y$. By Definition~\ref{def:pop-normalized-citations}, this is equivalent to conditioning on $C_f=(y-\mu_f)/s_f$.

By Assumption~\ref{assump:field-citation-gaussian}, the pair $(T_f,C_f)$ is jointly Gaussian. We can therefore apply Theorem~\ref{thm:winners-curse-citations} to the pair $(T_f,C_f)$, which yields
\begin{align*}
\mathbb{E}[T_f\mid C_f=c] = \rho_{C,f}^2\,c
\quad\text{for all } c\in\mathbb{R}.
\end{align*}
We substitute $c=(y-\mu_f)/s_f$ and we obtain \eqref{eq:ET-given-raw}.
\end{proof}

Equation~\eqref{eq:ET-given-raw} highlights two distinct sources of cross-field incomparability. First, the coupling factor $\rho_{C,f}^2$ can vary across fields. Second, the population scale $s_f$ controls how many raw citations correspond to one unit of the normalized score.

\subsection{A Population-Normalized Citation Exchange Rate}

We now define a simple exchange rate between citations in two fields.

\begin{definition}[Citation exchange rate]
\label{def:citation-exchange-rate}
Consider two fields $A$ and $B$. We define the marginal expected value per raw citation as
\begin{align*}
\kappa_f := \frac{\partial}{\partial y}\,\mathbb{E}[T_f\mid Y_f=y].
\end{align*}
We define the exchange rate from $A$ to $B$ as the number $\mathcal{E}_{A\to B}$ such that one additional citation in field $B$ has the same marginal expected value as $\mathcal{E}_{A\to B}$ additional citations in field $A$.
\end{definition}

\begin{theorem}[Closed form exchange rate]
\label{thm:exchange-rate}
Under the assumptions of Theorem~\ref{thm:raw-citation-conditional-value}, the marginal expected value per citation is constant in $y$ and equals
\begin{equation}
\label{eq:kappa-f}
\kappa_f = \frac{\rho_{C,f}^2}{s_f}.
\end{equation}
Consequently, the exchange rate from $A$ to $B$ equals
\begin{equation}
\label{eq:exchange-rate}
\mathcal{E}_{A\to B} = \frac{\kappa_B}{\kappa_A} = \frac{s_A}{s_B}\,\Bigl(\frac{\rho_{C,B}}{\rho_{C,A}}\Bigr)^2.
\end{equation}
\end{theorem}

\begin{proof}
We start from \eqref{eq:ET-given-raw}, which implies
\begin{align*}
\mathbb{E}[T_f\mid Y_f=y] = \rho_{C,f}^2\,\frac{y-\mu_f}{s_f}.
\end{align*}
Differentiating both sides with respect to $y$ yields
\begin{align*}
\kappa_f = \frac{\partial}{\partial y}\,\mathbb{E}[T_f\mid Y_f=y] = \frac{\rho_{C,f}^2}{s_f},
\end{align*}
which proves \eqref{eq:kappa-f}.

By Definition~\ref{def:citation-exchange-rate}, $\mathcal{E}_{A\to B}$ satisfies
\begin{align*}
\kappa_A\,\mathcal{E}_{A\to B} = \kappa_B.
\end{align*}
We solve for $\mathcal{E}_{A\to B}$ and we substitute \eqref{eq:kappa-f} for $A$ and $B$. This yields
\begin{align*}
\mathcal{E}_{A\to B}
= \frac{\rho_{C,B}^2/s_B}{\rho_{C,A}^2/s_A}
= \frac{s_A}{s_B}\,\Bigl(\frac{\rho_{C,B}}{\rho_{C,A}}\Bigr)^2,
\end{align*}
which is \eqref{eq:exchange-rate}.
\end{proof}

Equation~\eqref{eq:exchange-rate} separates population effects from truth-coupling effects. The ratio $s_A/s_B$ captures how raw citation volume differs across the two fields. The ratio $(\rho_{C,B}/\rho_{C,A})^2$ captures how strongly citations track latent value in each field.

\subsection{An Illustrative Comparison}

We close with a simple parameterization that illustrates how large cross-field exchange rates can be. We do not claim that the numbers below match any particular discipline. We instead choose them to represent two qualitative regimes, and we choose them so that each knob in the model has a clear interpretation.

The key quantity is
\begin{align*}
K_f = (1-q_f)^2 r_{C,f} = (1-q_f)^2\,\frac{\mathrm{Var}(\Delta_{C,f})}{\mathrm{Var}(T_f)}.
\end{align*}
A field can have a large $K_f$ for three distinct reasons. Verification can be rare, so $q_f$ is small. Signal can shrink, so $\mathrm{Var}(T_f)$ is small. Proxy amplification can grow, so $\mathrm{Var}(\Delta_{C,f})$ is large. Any of these effects can push the field into a low-coupling regime.

Consider two fields $A$ and $B$. We first set verification pressure. Under our queue model we have $q_f=\min\{1,1/\Lambda_f\}$. We set $\Lambda_A=10$ and $\Lambda_B=2$, so $q_A=0.1$ and $q_B=0.5$.

We next set the proxy-to-truth ratio $r_{C,f}$. We set $r_{C,A}=100$ and $r_{C,B}=5$. These are not estimates. We choose them to reflect two levels of noise dominance.
The choice $r_{C,A}=100$ means that proxy variance is two orders of magnitude larger than signal variance. This can arise from two moderate effects, namely signal shrinkage by a factor of $10$ together with proxy amplification by a factor of $10$.
The choice $r_{C,B}=5$ means that proxy variance is only a few times larger than signal variance, which represents a milder proxy regime.

For additional intuition, we can view $r_{C,f}$ as an effective number of independent proxy channels. If we write
\begin{align*}
\Delta_{C,f} = \sum_{j=1}^{m_f} Z_{f,j},
\end{align*}
where the variables $Z_{f,1},\ldots,Z_{f,m_f}$ are independent, have mean zero, are independent of $T_f$, and satisfy $\mathrm{Var}(Z_{f,j})=\mathrm{Var}(T_f)$, then
\begin{align*}
\mathrm{Var}(\Delta_{C,f})
&= \mathrm{Var}\Bigl(\sum_{j=1}^{m_f} Z_{f,j}\Bigr)
= \sum_{j=1}^{m_f}\mathrm{Var}(Z_{f,j})
= m_f\,\mathrm{Var}(T_f),
\end{align*}
so $r_{C,f}=m_f$. Under this interpretation, $r_{C,A}=100$ corresponds to $m_A=100$ effective proxy channels, while $r_{C,B}=5$ corresponds to $m_B=5$.

Given $q_A,q_B$ and $r_{C,A},r_{C,B}$, we obtain
\begin{align*}
K_A &= (1-q_A)^2 r_{C,A} = 0.9^2\cdot 100 = 81,\\
K_B &= (1-q_B)^2 r_{C,B} = 0.5^2\cdot 5 = \frac{5}{4}.
\end{align*}
Therefore
\begin{align*}
\rho_{C,A} &= \bigl(1+K_A\bigr)^{-1/2} = \frac{1}{\sqrt{82}} \approx 0.11,\\
\rho_{C,B} &= \bigl(1+K_B\bigr)^{-1/2} = \frac{2}{3} \approx 0.67.
\end{align*}

Finally, we set the population normalization. We take $s_f$ to be proportional to the number of active citers in field $f$. If field $A$ has five times as many active citers as field $B$, then $s_A/s_B=5$.

By Theorem~\ref{thm:exchange-rate}, one additional citation in field $B$ has the same marginal expected value as
\begin{align*}
\mathcal{E}_{A\to B}
&= \frac{s_A}{s_B}\,\Bigl(\frac{\rho_{C,B}}{\rho_{C,A}}\Bigr)^2\\
&= \frac{s_A}{s_B}\,\frac{1+K_A}{1+K_B}
= 5\cdot \frac{82}{9/4}
= \frac{1640}{9}
\approx 182
\end{align*}
additional citations in field $A$. This example is meant to build intuition. It shows that large exchange rates arise whenever one field operates in a low-coupling regime with $K_f$ much larger than $1$, while another field operates in a moderate-coupling regime with $K_f$ comparable to $1$.

%% file: appendices/09_empirical_agenda.tex
\section{Empirical Agenda}
\label{sec:empirical}

Our theory introduces quantities that we can estimate in real venues.
This section sketches a measurement program.
We emphasize feasibility over perfection.

\subsection{Estimating Verification Pressure $\vpressure$}

Equation~\eqref{eq:Lambda} defines $\vpressure = R\Ce/\Be$.
We can estimate each component.

We can estimate $R$ from submission counts or preprint inflow.
OpenReview provides submission volumes for many conferences, and arXiv provides category-level inflows.
We can refine $R$ by counting claims rather than papers, for example by extracting the number of headline results, benchmarks, or theorem statements.

We can estimate $\Be$ by counting reviewer assignments and by estimating available review time.
A coarse estimator multiplies the number of reviewers by an assumed hours-per-review.
A better estimator uses self-reported time, rebuttal length, or meta-review depth as proxies for attention.
Artifact evaluation programs provide another signal because they create explicit verification tasks.

We can estimate $\Ce$ using time-to-verify studies.
For computational papers, we can log the compute and engineering effort needed to reproduce a result from provided artifacts.
For theoretical work, we can log the time needed to check proofs or to find missing assumptions.
We can also decompose $\Ce=c/\kappa$ by measuring how often evidence resolves disputes.
For example, we can measure how often an artifact rerun changes a decision.
This is a direct proxy for $\kappa$.

\subsection{Estimating Signal Shrinkage and $\snr$}

We defined $\snr = \Var(\Delta)/\Var(\Truth)$ in \eqref{eq:snr}.
Direct truth measurement is hard, but we can estimate components.

We can estimate $\Var(\Delta)$ using instability and sensitivity metrics.
In machine learning, seed variance, hyperparameter sensitivity, and data leakage are measurable.
In experimental sciences, batch effects and measurement error provide analogs.
We can also estimate proxy noise by comparing claimed improvements with improvements that survive a standardized audit.

We can approximate $\Var(\Truth)$ using long-horizon outcomes.
One approach uses replication success.
Another uses post-publication robustness, such as whether a method remains competitive under later, stronger baselines.
A third uses ``disruption'' or citation topology metrics, although these remain indirect.

Saturation should show up as a decline in effect sizes on mature benchmarks.
For example, if a benchmark metric approaches a ceiling, then the year-to-year spread of improvements shrinks.
This reduces $\Var(\Truth)$ and increases $\snr$ even if proxy noise stays constant.

\subsection{Testing the Phase Transition Hypothesis}

The model predicts that proxy sovereignty becomes likely when $\vpressure$ is large and $\snr$ is large.
We can test this by comparing venues and subfields.

A practical strategy is to assemble a dataset of papers with artifacts, benchmark claims, and reproducibility outcomes.
We then measure proxies for $q$ and for $\snr$ and test whether these predict later instability, retractions, or rapid obsolescence.
We can also test whether interventions that reduce $\Ce$ or increase $\Be$ increase stability.

LLM-mediated writing creates a natural experiment.
If LLMs increase $R$ without increasing $\Be$, the theory predicts that truth-coupling should decline unless the community also scales verification-first mechanisms.
Recent measurements of AI-modified review text show that large-scale shifts already occur \citep{liang2024monitoring}.
We can build on such measurements to study downstream effects on coupling and on robustness.

\subsection{What Would Falsify Our Model}

The theory would be wrong or incomplete if we observe sustained high truth-coupling in regimes with high $\vpressure$ and high $\snr$.
Such cases would imply that communities found additional coupling mechanisms not captured in our abstraction.
Examples could include strong informal replication norms, industry testbeds, or tight proof cultures that effectively increase $\kappa$.
These cases would be scientifically valuable because they would suggest new design patterns for verification-first peer review.

%% file: appendices/appendix_case_studies.tex
\section{Case Studies: Mature Engineering Domains as Extreme Regimes}
\label{app:cases}

This appendix uses the $(R,\Ce,\Be)$ lens to interpret several mature engineering domains.
These examples are not meant as judgments about individuals.
They illustrate how a field's technical structure can create proxy sovereignty even when researchers act rationally.

\subsection{Wireless Communications and ``The Simulator as the World''}

Wireless communications provides an instructive extreme case.
The real-world object is a complex physical and socio-technical system.
To validate a claim about a new protocol at scale, we often need expensive testbeds, custom hardware, and deployment conditions.
This makes $c$ large.
At the same time, real environments are noisy and context-dependent, which can make $\kappa$ smaller than one.
As a result, $\Ce=c/\kappa$ can be very large.

The community therefore uses simulation proxies.
Simulations reduce $c$, but they can also reduce $\kappa$ if the simulator does not faithfully capture the deployment regime.
When the proxy becomes the evaluation standard, the community enters a proxy-sovereign loop.
Papers optimize for the simulator rather than for the real system.
As the domain matures and approaches physical limits, true improvements shrink.
This reduces $\Var(\Truth)$ and increases $\snr$, which further encourages proxy reliance.
In this sense, communications is an extreme realization of the phase diagram in Figure~\ref{fig:phase}.

\subsection{Autonomous Driving and Safety-Critical Claims}

Autonomous driving has high verification friction for a different reason.
Claims often concern rare events.
Rare-event validation requires large-scale testing or strong causal models.
Both increase $c$.
Because accidents are sparse, evidence can remain ambiguous, which reduces $\kappa$.
This raises $\Ce$.

In response, the community relies on proxy benchmarks and curated datasets.
This can be useful, but it also creates strong incentives to overfit to proxy distributions.
A verification-first review process would invest in stress tests and distribution shift audits that approximate the rare-event regime.
This directly targets $\kappa$ and reduces $\Ce$.

\subsection{Reinforcement Learning and Sensitivity-Driven Proxy Noise}

In reinforcement learning, results often exhibit large variance across random seeds and implementation details.
This increases $\Var(\Delta)$ even when the underlying idea is sound.
If reviewers lack bandwidth to run reruns, $q$ is small and the system selects based on noisy maxima, which increases $\snr$.
A verification-first intervention is straightforward in this domain: require seed sweeps and publish dispersion.
This reduces $\Var(\Delta)$ without demanding new theory.

\subsection{Materials Science and Long Experimental Cycles}

In experimental materials research, verification can require long fabrication cycles and specialized equipment.
This increases $c$ and also creates barriers to independent replication.
The result is high verification pressure even at moderate claim rates.
If publication incentives push $R$ upward, proxy sovereignty can emerge through the selection of results that appear strong in limited experimental settings.
Here, verification-first design may rely more on preregistration, shared protocols, and multi-site replication consortia.
These mechanisms increase $\Be$ and $\kappa$.

\subsection{A Unifying Interpretation}

Across these domains, we observe the same structure.
When verification is expensive and true gains shrink, communities adopt low-cost proxies to keep the system running.
Once proxies decide careers, optimization pressure shifts toward proxies.
This is a stable equilibrium for individuals, but it can reduce truth-coupling at the system level.
AI can either accelerate the drift, by scaling proxy optimization and claim production, or help resist it, by scaling decisive evidence.

%% file: appendices/appendix_limitations.tex
\section{Limitations}
\label{sec:limitations}

We proposed a phase transition view of scientific production driven by verification pressure and signal shrinkage.
This view yields clear design implications for AI-assisted review, but it also has limitations.

\subsection{Not all proxy use is bad}

Proxy signals are unavoidable.
Even in truth-sovereign disciplines, humans rely on heuristics to allocate attention.
Our claim is narrower.
Proxy sovereignty becomes harmful when proxy optimization dominates and when truth-coupling becomes low.
In that regime, Goodhart drift becomes likely and Theorem~\ref{thm:effort} predicts that honest truth effort becomes privately irrational.

\subsection{Equity and access}

Verification-first mechanisms can increase burdens on authors.
Compute-heavy audits may disadvantage low-resource groups.
We therefore need verification artifacts that are cheap and diagnostic.
We also need shared infrastructure, standardized audit suites, and community datasets.
These investments can increase $\Be$ without requiring each author to own large compute.

\subsection{The risk of verification theater}

A warning from the audit literature is that verification can itself become ritual \citep{power1997audit}.
If communities treat checklists as proxies, they can recreate proxy sovereignty in a new form.
Verification-first AI must keep evidence grounded.
A smoke test that runs and produces logs increases fidelity.
A checklist that asks for a justification paragraph can be gamed, as the NeurIPS 2024 experiment suggests \citep{goldberg2024checklist}.

\subsection{Modeling limitations}

Our model is intentionally simple.
We used a scalar truth variable and a scalar score.
Real peer review involves multi-dimensional judgments and also serves prioritization and feedback, in addition to social dynamics.
We also treated verification as random service with a fixed bandwidth.
In practice, committees route attention strategically and fields differ in their social structure.
We see our model as a minimal backbone that clarifies which quantities matter and how they interact.
Future work can extend it with network effects, strategic reviewer behavior, and heterogeneous paper types.